\PassOptionsToPackage{table}{xcolor}
\documentclass[11pt]{article}

\usepackage[preprint]{acl}

\usepackage{times}
\usepackage{latexsym}
\usepackage[T1]{fontenc}
\usepackage[utf8]{inputenc}
\usepackage{microtype}
\usepackage{graphicx}
\usepackage{amsmath}
\usepackage{amssymb}
\usepackage{amsthm}

\newtheorem{proposition}{Proposition}
\newtheorem{corollary}{Corollary}
\usepackage{booktabs}
\usepackage{multirow}
\usepackage{xspace}
\usepackage{subcaption}
\usepackage{bbm}
\usepackage{xcolor}

\newcommand{\acoer}{\textsc{Acoer}\xspace}
\newcommand{\thinktag}{\texttt{<think>}}
\newcommand{\thinktagend}{\texttt{</think>}}

\title{Beyond Penalizing Mistakes: Stabilizing Efficiency Training in Large Reasoning Models via Adaptive Correct-Only Rewards}

\author{
  \bfseries Jungseob Lee$^{1}$\thanks{Equal contribution.} \quad
  Seungyoon Lee$^{1}$\footnotemark[1] \quad
  Seongtae Hong$^{1}$ \\
  \bfseries Minhyuk Kim$^{1}$ \quad
  Chanjun Park$^{2}$\thanks{Corresponding authors.} \quad
  Heuiseok Lim$^{1}$\footnotemark[2] \\[4pt]
  $^{1}$Korea University \quad $^{2}$Soongsil University \\[3pt]
  \normalsize\ttfamily \{omanma1928, dltmddbs100, ghdchlwls123, mhkim0929, limhseok\}@korea.ac.kr \\
  \normalsize\ttfamily chanjun.park@ssu.ac.kr
}

\begin{document}
\maketitle

\begin{abstract}
Training large language models to reason efficiently is a critical challenge. While integrating length-penalizing rewards into Group Relative Policy Optimization (GRPO) aims to reduce verbosity, it frequently triggers reward collapse, severely degrading reasoning capabilities. Through a systematic evaluation of various reward configurations, we identify the root mechanism: GRPO's group normalization creates divergent advantages when incorrect answers receive continuous length penalties. Consequently, methods penalizing the length of incorrect answers are structurally prone to collapse under sustained optimization. Furthermore, restricting penalties exclusively to correct answers avoids this primary failure, but leaves the model susceptible to a stochastic collapse driven by response over-compression. To robustly prevent both failure modes, we propose \acoer (Adaptive Correct-Only Efficiency Reward). \acoer eliminates the structural penalty loop by isolating brevity bonuses to correct completions and prevents stochastic compression via dynamic budget normalization and control-loop penalty adjustments. Evaluated across diverse mathematical reasoning benchmarks, \acoer improves overall accuracy compared to the base model while reducing token generation by over 60\%, establishing a fundamentally stable approach for efficiency-aware optimization.
\end{abstract}

\section{Introduction}
\label{sec:intro}

Reasoning models such as DeepSeek-R1 \citep{guo2025deepseek} and Qwen3 \citep{yang2025qwen3} produce extended reasoning traces via reinforcement learning (RL), achieving strong performance on mathematical reasoning benchmarks. However, this capability incurs substantial computational cost, as these models frequently generate thousands of reasoning tokens per problem, even when the underlying query is solvable in far fewer tokens. Recent work has therefore begun to treat overthinking itself as an efficiency target, including survey work, reasoning-trace pruning, budget-control mechanisms, and adaptive decisions about when to think longer \citep{stop-overthinking, o1pruner, thinkprune, budgetthinker, thinkswitcher}.

A growing body of work addresses this efficiency problem by adding length-penalizing signals to Group Relative Policy Optimization (GRPO) \citep{shao2024deepseekmath} rewards \citep{lcpo, short-rl, fast-on-easy, laser, recut, alp, grpo-lead}. A critical, yet often under-explored, design choice in these methods is how to handle incorrect rollouts. When reward functions apply length penalties even to rollouts that ultimately yield the wrong answer, practitioners often encounter \emph{reward collapse}: after initial improvements, the model degenerates into producing outputs too short to reason effectively. Although widely recognized as a major hurdle, the exact mechanisms driving this collapse remain poorly understood. 

To investigate the root cause of this instability, we conduct a controlled comparison of diverse reward configurations trained under identical conditions. By systematically isolating reward components, we find that penalizing the length of incorrect rollouts is a primary reason. Specifically, in our evaluation, methods that apply continuous length penalties to incorrect rollouts consistently exhibited collapse, whereas methods that apply penalties exclusively to correct rollouts or use discrete binary thresholds remained largely stable. A sensitivity analysis further confirms this relationship, demonstrating that even a microscopically small penalty weight on wrong answers can trigger collapse.

Crucially, analyzing the optimization trajectories reveals that employing strictly correct-only rewards (i.e., entirely removing length penalties for incorrect rollouts) is a necessary but insufficient condition for absolute stability. While the model may remain stable under certain initializations, it remains vulnerable to sudden collapse midway through the optimization process, characterized by a severe drop in both problem-solving accuracy and generation length. This sensitivity to initialization indicates a deeper, stochastic structural instability within GRPO driven by the over-compression of correct rollouts, a vulnerability that static correct-only rewards cannot reliably prevent.

To address these dual collapse pathways, we propose \acoer (Adaptive Correct-Only Efficiency Reward). \acoer combines three synergistic mechanisms: applying length penalties exclusively to correct rollouts to eliminate the primary collapse loop associated with penalizing incorrect ones, adaptive budget normalization to track the model's evolving conciseness, and a control-loop schedule that dynamically adjusts efficiency pressure based on accuracy preservation. Empirically, \acoer successfully maintains a robust accuracy-efficiency tradeoff. On the MATH-500~\citep{math500}, it substantially reduces token generation by 62\% while simultaneously maintaining problem-solving accuracy comparable to the base model.

In summary, our contributions establish both the empirical and theoretical foundations of reward collapse. We demonstrate that applying continuous length penalties to incorrect rollouts strongly predisposes the optimization process to collapse. Unlike prior empirical approaches that utilize algebraically equivalent correct-only rewards without a mechanistic explanation \citep{short-rl}, we provide a formal theoretical analysis demonstrating that GRPO's group normalization inherently creates divergent advantages when length penalties are continuously applied to incorrect rollouts. By identifying both the primary penalty-driven pathway and the stochastic correct-answer compression pathway, our work offers a comprehensive understanding of GRPO instability and provides \acoer as a principled, adaptive solution for efficient reasoning.

\section{Background and Related Work}
\label{sec:background}

\subsection{Reasoning Models and the Efficiency Bottleneck}
Recent advancements in RL have enabled language models to produce explicit, internalized reasoning traces within \thinktag{}...\thinktagend{} tags before generating final answers, moving beyond traditional prompt-based chain-of-thought (CoT) techniques \citep{wei2022chain, guo2025deepseek, yang2025qwen3}. While these structured, self-correcting traces allow models to achieve strong performance on complex mathematical and logical benchmarks, standard RL objectives are typically outcome-oriented. Consequently, a model receives the same reward for a correct answer produced with 50 tokens as one produced with 2,000. This lack of a training signal for concise reasoning inadvertently encourages the generation of thousands of tokens, establishing a severe efficiency bottleneck during inference.

\subsection{Efficient Reasoning via Length Penalties}
To mitigate this efficiency bottleneck, a rapidly growing line of research seeks to incorporate length-penalizing signals into RL rewards \citep{lcpo, dast, short-rl, fast-on-easy, alp, thinking-fast-right, grpo-lead, laser, recut}, while complementary trace-shortening approaches prune, remove, or otherwise shorten reasoning traces \citep{o1pruner, thinkprune, nowait}. Among reward-shaping approaches, Short-RL \citep{short-rl} has empirically demonstrated that applying length penalties exclusively to correct answers (algebraically equivalent to setting our incorrect-answer penalty to zero, $\beta{=}0$) can improve stability, but the mechanistic reasons behind why these correct-only rewards resist collapse remain an open question. Furthermore, characterizing the specific optimization conditions under which even correct-only rewards might experience stochastic instability represents a critical gap in the field. Orthogonal approaches have focused on demonstrating that standard optimization inherently compresses verbosity under specific optimal conditions \citep{conciser}, but explicit reward shaping remains the most direct and widely used intervention.

\subsection{GRPO and Reward Collapse}
The primary algorithm used to optimize these reasoning models, including those employing the aforementioned efficiency rewards, is GRPO \citep{shao2024deepseekmath}. For a given prompt $x$, GRPO samples a group of $G$ completions $\{y_1, \ldots, y_G\}$ and updates the policy based on advantages computed via group-normalized rewards:
\begin{equation}
    \hat{A}_i = \frac{r_i - \text{mean}(\{r_j\}_{j=1}^G)}{\text{std}(\{r_j\}_{j=1}^G)}
    \label{eq:grpo_advantage}
\end{equation}
However, a critical structural vulnerability in this formulation emerges when the standard deviation of rewards within a group approaches zero ($\text{std}(\{r_j\}) = 0$), nullifying the gradient. This pathology becomes highly active specifically when length-penalizing signals are introduced to combat verbosity, frequently precipitating a phenomenon known as \emph{reward collapse}. When reward collapse occurs, the RL agent optimizes toward a pathological minimum rather than the intended objective, characterized by a sudden and severe drop in both token generation length and problem-solving accuracy. While some prior works modify the normalization process itself \citep{drpo} to address instability, our work serves as a complementary, reward-side analysis. Rather than altering the core optimization algorithm, we rigorously identify which penalty designs remain mathematically safe under standard GRPO normalization, ultimately uncovering the hidden mechanisms that trigger both structural and stochastic collapse pathways.

\section{Reward Design Space and Experimental Setup}
\label{sec:reward_design}

To systematically diagnose the mechanisms driving reward collapse during efficiency training, we construct a unified evaluation framework. Rather than evaluating isolated methods with disparate implementations, we formalize a generalized reward space. This allows us to strictly control individual components, such as directionality and magnitude, and observe their direct impact on optimization stability. We evaluate 11 distinct reward configurations to isolate the specific triggers of instability.

\subsection{A Unified Framework for Efficiency Rewards}
\label{ssec:formulations}

To strictly isolate the cause of reward collapse, we define a generalized formulation that decouples the length penalties applied to correct reasoning traces from those applied to incorrect ones. The base reward function is defined as:

\begin{equation}
    r = c + c \cdot \alpha \cdot f(\ell) - (1 - c) \cdot \beta \cdot f(\ell) + r_{\text{format}}
    \label{eq:general_reward}
\end{equation}

Here, $c \in \{0, 1\}$ represents the outcome correctness of the final answer, and $\ell$ denotes the number of generated CoT tokens. The term $f(\ell)$ is a length-dependent function, such as a linear decay $(1 - \ell/L)$ where $L$ is the maximum token limit. Crucially, $\alpha \geq 0$ dictates the magnitude of the brevity bonus explicitly reserved for \textbf{correct} answers, while $\beta \geq 0$ controls the magnitude of the length penalty applied to \textbf{incorrect} answers. Finally, $r_{\text{format}}$ is a discrete format reward, strictly yielding $1$ if the model produces valid \thinktag{} and \texttt{\textbackslash boxed\{\}} tags, and $0$ otherwise. By manipulating $\alpha$ and $\beta$, we can precisely construct both existing baselines and targeted ablations to observe when the optimization process degenerates. 

\subsection{Taxonomy of Reward Configurations}
\label{ssec:taxonomy}

Using this generalized formulation, we categorize the 11 evaluated configurations along four critical design dimensions: \textbf{directionality} (whether length signals are applied to correct answers, incorrect answers, or both), \textbf{adaptivity} (fixed versus dynamically adjusted parameters), \textbf{boundedness} (whether the length gradient $|\partial r / \partial \ell|$ is bounded), and \textbf{signal form} (linear, exponential, or discrete binary).

Among these, \textbf{directionality} is the most critical distinction in our evaluation. Table~\ref{tab:key_formulas} illustrates how representative methods map to our framework based on this dimension. Specifically, methods with $\beta > 0$ apply continuous length signals to incorrect answers in addition to correct ones (\textit{Bidirectional}), while methods with $\beta = 0$ restrict length signals exclusively to correct answers (\textit{Correct-only}). In method names, acc denotes accuracy-only reward and LP denotes length penalty.

\begin{table}[t]
\centering
\resizebox{\columnwidth}{!}{%
\begin{tabular}{lllc}
\toprule
\textbf{Category} & \textbf{Method} & \textbf{Reward ($r$)} & \textbf{$\beta$} \\
\midrule
\multirow{2}{*}{Correct-Only} 
& GRPO (acc) & $c$ & 0 \\
& $\beta{=}0$ (Ablation) & $c + 0.3c(1{-}\ell/L)$ & 0 \\
\midrule
\multirow{2}{*}{Bidirectional} 
& $\beta{=}0.10$ (Ablation) & $c + 0.3c(1{-}\ell/L) - 0.1(1{-}c)(1{-}\ell/L)$ & 0.1 \\
& GRPO+LP & $c - 0.3\ell/L$ & 0.3 \\
\midrule
\bottomrule
\end{tabular}
}
\caption{Reward formulations for representative configurations evaluated in our diagnostic setup (format reward $r_{\text{format}}$ omitted for clarity). $L$ represents the maximum token limit, and LP denotes length penalty.}
\label{tab:key_formulas}
\end{table}

\subsection{Experimental Setup}
\label{ssec:experimental_setup}

To ensure that any observed instability stems strictly from the reward formulation rather than confounding training dynamics, all configurations are trained under completely identical optimization conditions. We train Qwen3-1.7B \citep{yang2025qwen3} on the NuminaMath-TIR dataset using standard GRPO with a group size of $G{=}16$ for 1,200 steps. 

The primary evaluation is conducted on the MATH-500 benchmark \citep{math500}, with out-of-domain transfer and robustness assessed on MATH-Hard \citep{math500}, American Invitational Mathematics Examination (AIME) 2025 \citep{aime2025}, and OlympiadBench \citep{he2024olympiadbench}. Further details regarding datasets and training hyperparameters are provided in Appendix~\ref{sec:appendix_hyperparams}.

\section{Diagnosing Reward Collapse}
\label{sec:diagnosis}

To rigorously investigate the mechanisms driving reward collapse, we first establish a quantitative criterion based on preliminary observations: a training regime is considered collapsed when the mean reasoning tokens fall below a task-dependent threshold (e.g., 500 tokens for MATH-500) and accuracy degrades by at least 5\% from its peak for $\geq$200 consecutive steps. Using this definition, we evaluate our generalized reward space (\S\ref{sec:reward_design}) to isolate the triggers of instability.

\subsection{Empirical Analysis of Length Penalties}
\label{ssec:empirical_analysis}

We first examine the impact of continuous length penalties on incorrect answers by training GRPO with a fixed correct-answer brevity bonus ($\alpha = 0.3$) across varying incorrect-answer penalty weights ($\beta \in \{0, 0.01, 0.05, 0.10\}$). 

As shown in Table~\ref{tab:beta0_trajectory}, the sensitivity analysis reveals a sharp phase boundary: all configurations with $\beta > 0$ systematically collapse under sustained optimization. Strikingly, even a marginal penalty ($\beta=0.01$), which is $30\times$ weaker than the correct-answer signal, suffices to trigger degeneration by step 800. This confirms that instability stems strictly from the presence of a non-zero length gradient rather than its magnitude. While larger $\beta$ values generally accelerate length concentration in expectation, the exact onset of collapse exhibits non-monotonic timing due to early-training trajectory fluctuations.

In contrast, setting $\beta=0$ (correct-only penalization) eliminates this structural failure. Across experiments in two random seeds in Table~\ref{tab:beta0_trajectory}, the first run maintains robust accuracy (84.4\%) and achieves a 59\% token reduction by step 1,200 without significant degradation. However, a critical vulnerability persists in another run: a different initialization of the strictly correct-only ($\beta=0$) configuration still collapses midway through training. This demonstrates that while correct-only rewards are necessary to prevent the primary failure mode, they are insufficient for absolute stability, leaving the model susceptible to stochastic degeneration.

\begin{table}[t]
\centering
\resizebox{\columnwidth}{!}{%
\setlength{\tabcolsep}{3pt}
\begin{tabular}{@{} l | cc ccc @{}}
\toprule
& \multicolumn{2}{c}{\textbf{$\beta=0$}} & \multicolumn{3}{c}{\textbf{$\beta > 0$}} \\
\cmidrule(lr){2-3} \cmidrule(lr){4-6}
\textbf{Step} & \textbf{Stable} & \textbf{Collapsed} & \textbf{0.01} & \textbf{0.05} & \textbf{0.10} \\
\midrule
Base & \multicolumn{5}{c}{88.8 / 5,553} \\
200  & 86.2/5142 & 87.8/4753 & \textbf{89.2}/4666 & 88.0/4793 & 88.2/5111 \\
400  & \textbf{88.6}/4678 & 87.8/2778 & 87.8/2879 & 88.0/3366 & 86.8/3874 \\
600  & 88.0/3937 & 80.6$^\dag$/1634 & 76.0$^\dag$/1598 & 83.8/1823 & \textbf{87.6}/2912 \\
800  & 87.0/2977 & 67.8$^\dag$/970 & 65.4$^\dag$/1005 & 62.8$^\dag$/832 & 82.8/1812 \\
1000 & 87.8/2571 & 55.8/984 & 62.4/881 & 56.0/594 & 67.8$^\dag$/1392 \\
1200 & 84.4/2255 & 59.6/591 & 65.8/901 & 61.8/991 & 66.2/1453 \\
\bottomrule
\end{tabular}
}
\caption{$\beta$ sensitivity and multi-seed analysis. Format: accuracy (\%) / mean total tokens. $^\dag$Collapse onset.}
\label{tab:beta0_trajectory}
\end{table}

\subsection{Mechanistic Analysis of Collapse Pathways}
\label{sec:mechanism}

We formalize these empirical findings by identifying two distinct collapse pathways driven by GRPO's group normalization mechanism (Equation~\ref{eq:grpo_advantage}).

\begin{figure*}[t]
\centering
\includegraphics[width=\textwidth]{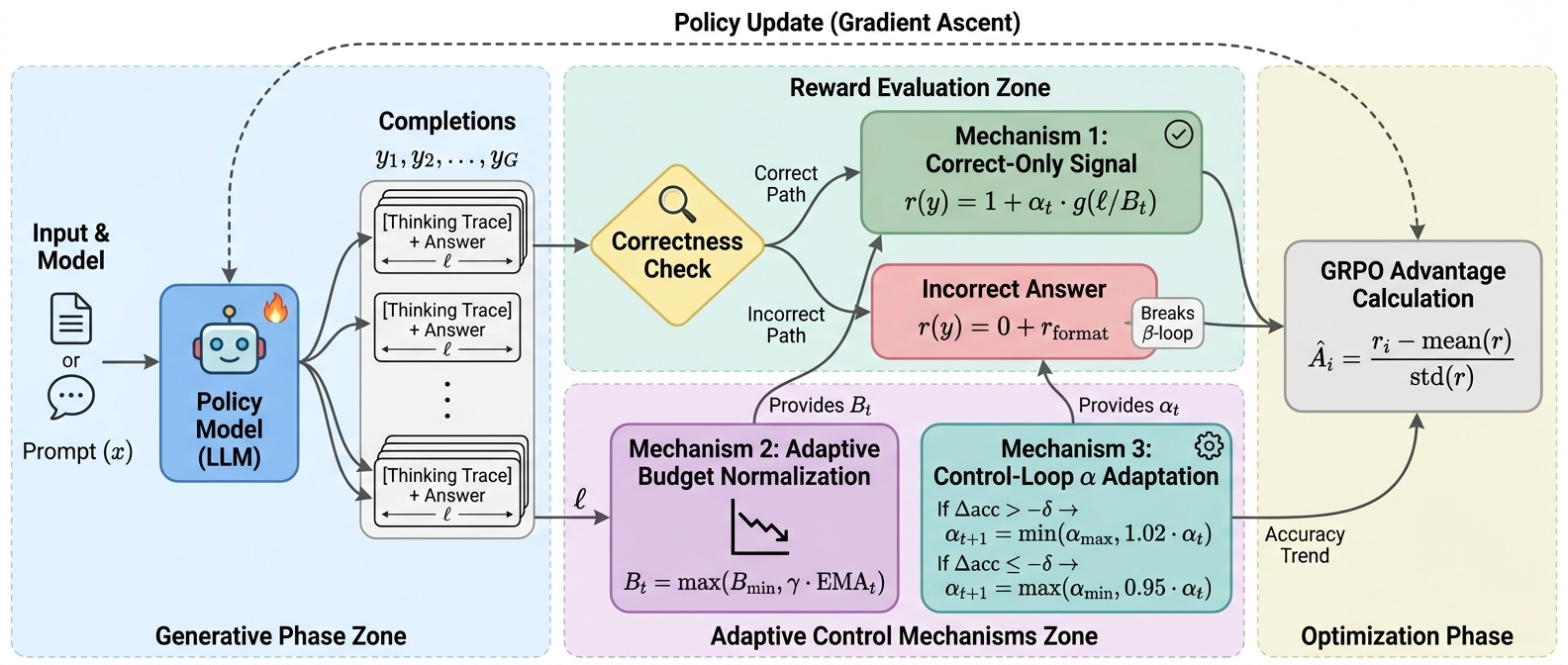}
\caption{\acoer framework overview. Three mechanisms target the identified collapse pathways: (1)~Correct-only signal ($\beta{=}0$) eliminates the $\beta$-loop (Pathway 1); (2)~Adaptive budget normalization tracks the exponential moving average (EMA) of correct-answer lengths, preventing runaway feedback; (3)~Control-loop $\alpha$ adaptation backs off efficiency pressure when accuracy drops, preventing Pathway 2 (correct-answer compression).}
\label{fig:acoer_framework}
\end{figure*}

\paragraph{Pathway 1: The \texorpdfstring{$\beta$}{beta}-Loop (Structural Collapse).} 
To build theoretical intuition for why continuous incorrect-answer length signals ($\beta > 0$) systematically destroy optimization, we analyze the gradients of the GRPO advantage function. We identify that this structural collapse is driven by two complementary forces: the accumulation of biased gradients in mixed groups, and intermittent explosive updates in homogeneous groups.

\textbf{1) Constant Downward Pressure in Mixed Groups.} 
In practical settings with a group size of $G \ge 16$, a generated batch typically contains a mix of correct ($r \approx 1$) and incorrect ($r \approx 0$) responses. In such mixed groups, the standard deviation of the rewards $\sigma_r$ is dominated by the large gap between correct and incorrect outcomes. Assuming an incorrect sample $i$ receives a length-penalized reward $r_i = -\beta \ell_i$ (where $\beta > 0$ is the penalty coefficient and $\ell_i$ is the token length), its normalized advantage $\hat{A}_i$ is approximated by:
\begin{equation}
\hat{A}_i = \frac{r_i - \bar{r}}{\sigma_r} \approx \frac{-\beta \cdot \ell_i - \bar{r}}{C}
\end{equation}
where $\bar{r}$ denotes the mean reward of the group and $C > 0$ is a strictly positive constant (e.g., $C \approx 0.4 \sim 0.5$) representing the typical reward variance in mixed batches. While the large denominator $C$ prevents the advantage from diverging in a single step, the gradient with respect to the generated length remains strictly negative:
\begin{equation}
\frac{\partial \hat{A}_i}{\partial \ell_i} \approx \frac{-\beta}{C} < 0
\end{equation}
Although this gradient is squashed by the constant $C$, it provides a \emph{constant downward pressure}. Over thousands of RL optimization steps, the relentless accumulation of these biased, negative gradients continuously penalizes the length of incorrect reasoning paths, eventually driving the policy to truncate reasoning entirely.

\textbf{2) Intermittent Explosive Divergence.} 
While mixed groups provide steady erosion, homogeneous groups (where all $G$ completions are incorrect) introduce severe instability. To formally demonstrate this extreme case, we consider a simplified setting where $G=2$ and both outputs are incorrect.

\begin{proposition}[Simplified Homogeneous Divergence]
Let $G=2$ with generated lengths $\ell_1 \neq \ell_2$. If both outputs are incorrect, the reward is solely determined by the length penalty: $r_i = -\beta \ell_i$. The absolute normalized advantage simplifies to:
\begin{equation}
|\hat{A}_i| = \frac{|r_i - \bar{r}|}{\sigma_r} = 1
\end{equation}
However, if the policy learns to output uniform lengths ($\ell_1 \approx \ell_2$), the variance $\sigma_r \to 0$. Any infinitesimal difference in length $\Delta \ell$ causes the advantage magnitude to diverge ($|\hat{A}| \to \infty$).
\end{proposition}
\vspace{-1mm}
\noindent\textit{Proof in Appendix~\ref{sec:appendix_proof}.}

In this homogeneous scenario, the normalization denominator vanishes, turning the length penalty into a runaway positive feedback loop. Even at moderate accuracy levels (${\sim}70\%$), statistical fluctuations dictate that a non-trivial fraction of groups will contain predominantly incorrect answers (see Appendix~\ref{sec:appendix_group_stats}). When these homogeneous batches occur, they inject explosive gradient spikes that violently pull the policy toward zero length.

\textbf{Synthesis.} The combination of these two forces explains the empirical inevitability of the $\beta$-loop: mixed groups provide a steady, slow-burning accumulation of length-reducing gradients, while intermittent homogeneous groups act as explosive catalysts. This explains why even a diminutive penalty ($\beta=0.01$) is sufficient to trigger collapse. Furthermore, it elucidates why certain bidirectional penalty methods remain stable if they use discrete signals:

\begin{corollary}[Binary immunity]
Methods utilizing discrete threshold penalties employ a piecewise constant function where $\partial g/\partial \ell = 0$ almost everywhere.
Without continuous gradients, the feedback loop cannot form, rendering binary penalties stable regardless of directionality.
\end{corollary}

\paragraph{Pathway 2: Stochastic Correct-Answer Compression.} 
The secondary collapse observed under $\beta=0$ exposes an intrinsic vulnerability within the correct-answer subgroup.
Even when incorrect penalties are removed entirely, the $\alpha$-term continuously differentiates short-correct from long-correct outputs.
If a particular initialization or early training dynamic causes the correct-answer length distribution to over-compress, the identical divergence logic from Proposition~1 applies to the correct-answer subgroup, triggering a stochastic collapse.

\paragraph{Implications.}
This dual vulnerability establishes a clear premise: while correct-only rewards ($\beta=0$) are a mathematical necessity for eliminating the structural $\beta$-loop, their static application cannot reliably prevent correct-answer over-compression.
This limitation necessitates dynamic parameter adjustment, which we formalize in the \acoer framework (\S\ref{sec:acoer}).

\begin{table*}[t]
\centering
\small
\setlength{\tabcolsep}{4pt}
\resizebox{\textwidth}{!}{%
\begin{tabular}{l cc cc cc cc}
\toprule
& \multicolumn{2}{c}{\textbf{MATH-500}} & \multicolumn{2}{c}{\textbf{MATH-Hard}} & \multicolumn{2}{c}{\textbf{AIME 2025}} & \multicolumn{2}{c}{\textbf{OlympiadBench}} \\
\cmidrule(lr){2-3} \cmidrule(lr){4-5} \cmidrule(lr){6-7} \cmidrule(lr){8-9}
\textbf{Method} & \textbf{Acc (\%)} & \textbf{\# Tokens} & \textbf{Acc (\%)} & \textbf{\# Tokens} & \textbf{Acc (\%)} & \textbf{\# Tokens} & \textbf{Acc (\%)} & \textbf{\# Tokens} \\
\midrule
\multicolumn{9}{c}{\textit{Baseline}} \\
Base   & 88.8 & 5,553 & 76.4 & 7,958 & 30.0 & 13,298 & 55.3 & 9,579 \\
\midrule
\rowcolor{gray!15}
\acoer              & \underline{88.4} & \textbf{2,134} \,{\scriptsize ($\downarrow$62\%)} & \textbf{78.1} & \textbf{3,509} \,{\scriptsize ($\downarrow$56\%)} & \textbf{36.7} & \underline{8,922} \,{\scriptsize ($\downarrow$33\%)} & \underline{55.3} & \underline{5,177} \,{\scriptsize ($\downarrow$46\%)} \\
GRPO-acc$^*$         & \textbf{88.8} & 4,091 \,{\scriptsize ($\downarrow$26\%)} & \underline{77.6} & 6,277 \,{\scriptsize ($\downarrow$21\%)} & \underline{33.3} & 11,836 \,{\scriptsize ($\downarrow$11\%)} & \textbf{56.2} & 7,982 \,{\scriptsize ($\downarrow$17\%)} \\
$\beta{=}0$ (s42)   & 84.4 & \underline{2,255} \,{\scriptsize ($\downarrow$59\%)} & 70.9 & 4,149 \,{\scriptsize ($\downarrow$48\%)} & 10.0 & \textbf{8,191} \,{\scriptsize ($\downarrow$38\%)} & 33.2 & \textbf{4,767} \,{\scriptsize ($\downarrow$50\%)} \\
\midrule
GRPO+LP              & 68.0 & 561 \,{\scriptsize ($\downarrow$90\%)} & 40.2 & 965 \,{\scriptsize ($\downarrow$88\%)} & 0.0 & 1,810 \,{\scriptsize ($\downarrow$86\%)} & 32.5 & 1,222 \,{\scriptsize ($\downarrow$87\%)} \\
\quad @400$^\dag$    & 86.8 & 3,279 \,{\scriptsize ($\downarrow$41\%)} & 77.8 & 5,186 \,{\scriptsize ($\downarrow$35\%)} & 30.0 & 11,643 \,{\scriptsize ($\downarrow$12\%)} & 54.2 & 7,150 \,{\scriptsize ($\downarrow$25\%)} \\
GRPO-LEAD            & 55.8 & 1,005 \,{\scriptsize ($\downarrow$82\%)} & 36.1 & 1,716 \,{\scriptsize ($\downarrow$78\%)} & 10.0 & 4,594 \,{\scriptsize ($\downarrow$65\%)} & 30.1 & 2,358 \,{\scriptsize ($\downarrow$75\%)} \\
\quad @200$^\dag$    & 88.8 & 4,718 \,{\scriptsize ($\downarrow$15\%)} & 77.9 & 7,110 \,{\scriptsize ($\downarrow$11\%)} & 20.0 & 13,396 \,{\scriptsize ($\uparrow$1\%)} & 56.4 & 8,780 \,{\scriptsize ($\downarrow$8\%)} \\
ReCut                & 65.6 & 943 \,{\scriptsize ($\downarrow$83\%)} & 46.8 & 1,859 \,{\scriptsize ($\downarrow$77\%)} & 10.0 & 4,853 \,{\scriptsize ($\downarrow$64\%)} & 31.3 & 2,401 \,{\scriptsize ($\downarrow$75\%)} \\
\quad @400$^\dag$    & 87.6 & 2,487 \,{\scriptsize ($\downarrow$55\%)} & 73.6 & 4,191 \,{\scriptsize ($\downarrow$47\%)} & 30.0 & 9,933 \,{\scriptsize ($\downarrow$25\%)} & 55.0 & 5,888 \,{\scriptsize ($\downarrow$39\%)} \\
\bottomrule
\end{tabular}%
}
\caption{Main results across four benchmarks. \# Tokens denotes the mean total generated tokens (reasoning and final response). Percentages in parentheses ($\downarrow$/$\uparrow$) indicate the token change relative to the Base model. Stable methods report metrics at step 1,200; collapsed methods at step 1,200 and at their pre-collapse peak ($^\dag$). $^*$GRPO-acc denotes accuracy-only GRPO and is reported at step 1,000 (peak within evaluation window); LP denotes length penalty.}
\label{tab:main_results}
\end{table*}

\section{\acoer: Adaptive Solution}
\label{sec:acoer}

To address the dual collapse vulnerabilities identified, we propose \acoer (Adaptive Correct-Only Efficiency Reward). Building upon the generalized formulation from Section~\ref{sec:reward_design}, \acoer introduces dynamic parameter control through three complementary mechanisms (illustrated in Figure~\ref{fig:acoer_framework}). At optimization step $t$, the reward for a generated response $y$ of length $\ell$ is defined as:
\begin{equation}
r(y) = \begin{cases}
1 + \alpha_t \cdot g(\ell / B_t) & \text{if correct} \\
0 + r_{\text{format}} & \text{if incorrect}
\end{cases}
\label{eq:acoer}
\end{equation}
where $g(x) = \log(1 + k(1{-}x))/\log(1{+}k)$ with $k{=}5$ serves as a smoothed scaling function. The components $\beta{=}0$, $B_t$, and $\alpha_t$ correspond to the following three mechanisms.

\subsection{Mechanism 1: Correct-Only Signal (\texorpdfstring{$\beta{=}0$}{beta=0})}
To mitigate the structural divergence caused by penalizing incorrect reasoning (Pathway 1), \acoer enforces a correct-only efficiency signal. By explicitly setting $\beta{=}0$, incorrect completions receive no length penalty, isolating the brevity bonus entirely to successful problem-solving trajectories. This design removes the primary catalyst for the $\beta$-loop.

\subsection{Mechanism 2: Adaptive Budget Normalization (\texorpdfstring{$B_t$}{B\_t})}
Standard efficiency rewards normalize token lengths against a fixed maximum limit $L$. However, as the model learns to generate shorter responses, fixed normalization yields progressively larger efficiency bonuses for minor token reductions, potentially encouraging over-compression. To mitigate this, \acoer replaces $L$ with a dynamic budget $B_t$ that tracks the exponential moving average (EMA) of correct-answer reasoning lengths. Specifically, $B_t = \max(B_{\min}, \gamma \cdot \text{EMA}_t)$, where $\gamma{=}0.85$ is a scaling factor and $B_{\min}$ prevents the budget from decaying to zero. By proportionally shrinking the budget as average lengths decrease, this mechanism stabilizes the $\ell/B_t$ ratio and helps moderate runaway feedback.

\subsection{Mechanism 3: Control-Loop Penalty Adaptation (\texorpdfstring{$\alpha_t$}{alpha\_t})}
While Mechanisms 1 and 2 address the structural $\beta$-loop and steady-state normalization, the model remains susceptible to stochastic correct-answer compression (Pathway 2) if efficiency pressure compromises reasoning accuracy. To safeguard against this, \acoer dynamically adjusts the penalty weight $\alpha_t$ based on accuracy preservation:
\begin{equation}
\alpha_{t+1} = \begin{cases}
\min(\alpha_{\max}, 1.02 \cdot \alpha_t) & \text{if } \Delta\text{acc} > -\delta \\
\max(\alpha_{\min}, 0.95 \cdot \alpha_t) & \text{if } \Delta\text{acc} \leq -\delta
\end{cases}
\end{equation}
where $\Delta\text{acc}$ is the change in moving-average accuracy measured over a 100-step sliding window, and $\delta{=}0.02$ acts as a tolerance threshold. Beginning with a conservative initial weight ($\alpha_0{=}0.02$), the system gradually increases efficiency pressure. If accuracy drops beyond the threshold, $\alpha_t$ decays more rapidly (5\% decrease) compared to its growth rate (2\% increase). This asymmetric adjustment facilitates recovery from incipient stochastic collapse while keeping the penalty weight bounded within $[\alpha_{\min}, \alpha_{\max}]$.

\section{Results}
\label{sec:results}

Table~\ref{tab:main_results} presents peak results across four benchmarks within the 1,200-step evaluation window.

\paragraph{Structural Collapse Dichotomy.}
Among the evaluated methods employing continuous $\beta > 0$, the collapse rate is 100\% (8/8). In contrast, among correct-only and binary methods, the rate is 17\% (1/6, with only the static $\beta{=}0$ ablation collapsing). This dichotomy, supported by the sensitivity analysis in Section~\ref{ssec:empirical_analysis}, suggests that continuous length signals on incorrect answers are a primary driver of structural collapse.

\paragraph{Performance on MATH-500.}
On the MATH-500 benchmark, \acoer maintains comparable accuracy while substantially reducing token usage relative to the base model. At step 1200, it achieves a peak accuracy of 88.4\% with a mean of 2,134 total tokens, representing a 62\% reduction from the base model. The token allocation suggests task-adaptive behavior: the model remains concise on solvable problems but expends more effort on harder queries.

\begin{figure}[t]
\centering
\includegraphics[width=0.9\columnwidth]{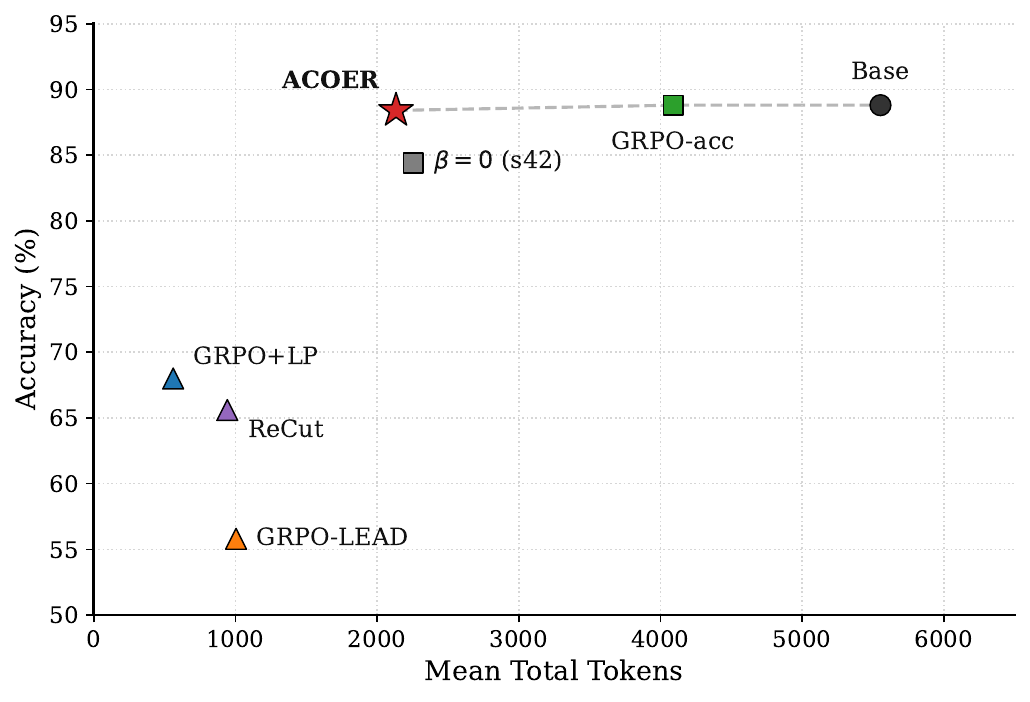}
\caption{Accuracy vs.\ efficiency for stable methods on MATH-500.}
\label{fig:pareto}
\end{figure}

\begin{figure*}[t]
\centering
\includegraphics[width=0.9\textwidth]{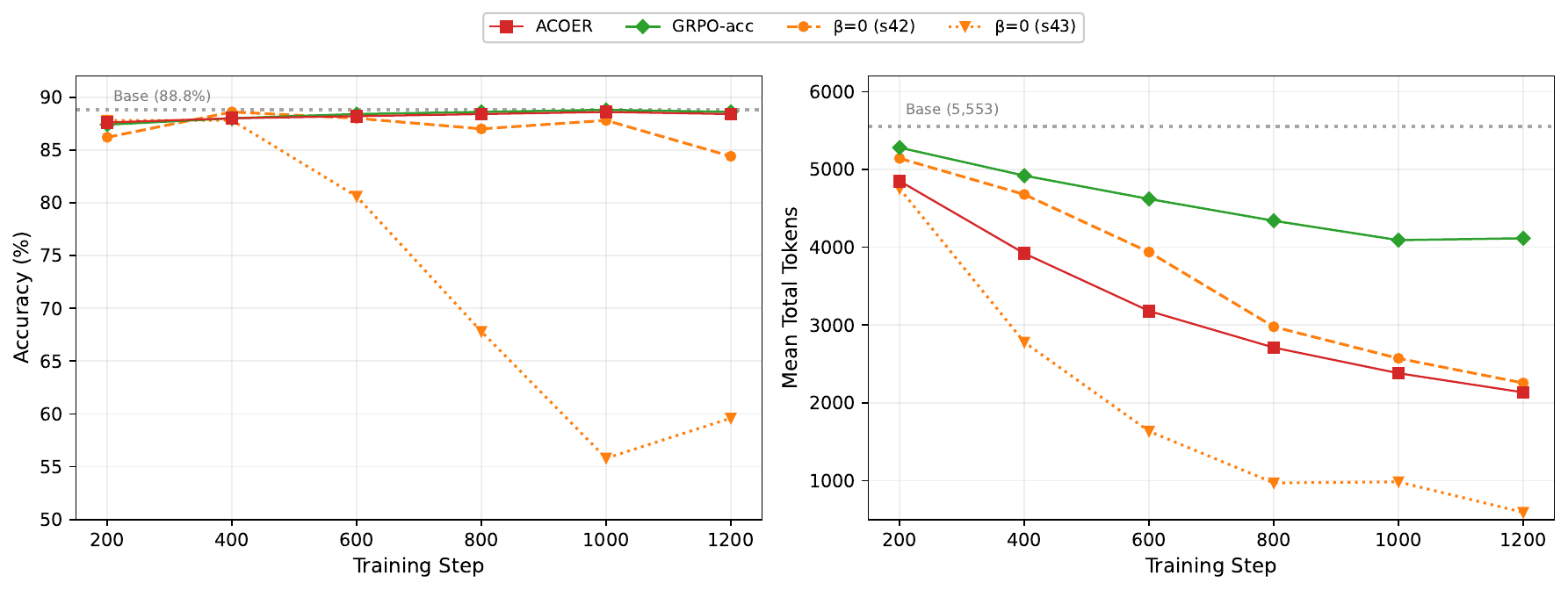}
\caption{Training dynamics for representative methods on MATH-500.}
\label{fig:main_trajectory}
\end{figure*}

\paragraph{Comparison with Collapsed Methods.}
In Figure~\ref{fig:main_trajectory}, other collapsed methods (e.g., $\beta{=}0.05$: 55.2\% accuracy / 1,119 tokens) exhibit severe length concentration, compressing both correct and incorrect answers uniformly. Consequently, their accuracy on the most challenging subsets (Level 5) drops significantly. \acoer mitigates this degradation by maintaining a broader dynamic range in its length distribution.

\paragraph{Per-Difficulty Analysis}
\label{sec:per_difficulty}

To better contextualize \acoer's efficiency, we analyze its performance across the five difficulty levels of MATH-500 in Table~\ref{tab:per_difficulty}. 

\paragraph{Accuracy Gains Across Complexities.}
The base model exhibits a predictable accuracy drop from 97.7\% on Level 1 to 78.4\% on Level 5. \acoer @1200 improves accuracy on Levels 2, 3, and 4, while maintaining equivalent performance on Level 1. Notable gains are observed on medium-difficulty problems (Level 3: +1.0pp, from 91.4\% to 92.4\%) and hard problems (Level 4: +0.8pp, from 89.8\% to 90.6\%), though accuracy on the most difficult problems (Level 5) experiences a slight decrease (from 78.4\% to 74.6\%).

\paragraph{Difficulty-Adaptive Token Allocation.}
Rather than applying uniform truncation, \acoer adapts its token usage to problem difficulty. On Level 1 problems, total tokens are reduced by 78.6\% (2,814 to 603 tokens). In contrast, for Level 5 problems, \acoer allocates a considerably larger reasoning budget of 3,986 tokens (a 52.6\% reduction). This confirms that \acoer preserves necessary computational depth for complex tasks while streamlining simple queries.

\begin{table}[t]
\centering
\footnotesize
\begin{tabular}{llccc}
\toprule
\textbf{Level} & \textbf{Method} & \textbf{Acc (\%)} & \textbf{\# Tokens} & \textbf{$\Delta$Tok} \\
\midrule
\multirow{2}{*}{1 (Easy)}
& Base & \textbf{97.7} & 2,814 & -- \\
& \acoer & \textbf{97.7} & 603 & $-$78.6\% \\
\midrule
\multirow{2}{*}{2}
& Base & 95.6 & 3,266 & -- \\
& \acoer & \textbf{96.7} & 845 & $-$74.1\% \\
\midrule
\multirow{2}{*}{3 (Med)}
& Base & 91.4 & 4,374 & -- \\
& \acoer & \textbf{92.4} & 1,580 & $-$63.9\% \\
\midrule
\multirow{2}{*}{4 (Hard)}
& Base & 89.8 & 6,062 & -- \\
& \acoer & \textbf{90.6} & 2,071 & $-$65.8\% \\
\midrule
\multirow{2}{*}{5 (V.Hard)}
& Base & \textbf{78.4} & 8,404 & -- \\
& \acoer & 74.6 & 3,986 & $-$52.6\% \\
\bottomrule
\end{tabular}
\caption{Per-difficulty comparison: \acoer @1200 vs.\ base model.}
\label{tab:per_difficulty}
\end{table}

\section{Conclusion}
\label{sec:conclusion}

In this work, we identify continuous incorrect-answer length signals as a primary structural driver of reward collapse in GRPO efficiency training. We demonstrate that penalizing the length of incorrect reasoning paths systematically degrades model performance (Pathway 1). While a correct-only length signal ($\beta=0$) removes this primary failure mode, models remain vulnerable to stochastic instability driven by over-compression of correct answers (Pathway 2).

To address these dual vulnerabilities, we propose \acoer, an adaptive framework governed by three principles: correct-only length signals to prevent structural collapse, dynamic budget normalization to mitigate runaway feedback, and accuracy-gated efficiency pressure to ensure safe optimization. Across four mathematical reasoning benchmarks, \acoer effectively balances accuracy and efficiency. On MATH-500, it maintains accuracy comparable to the base model while reducing token generation by 62\%. Crucially, these efficiency gains generalize to more challenging datasets, confirming that \acoer adaptively scales reasoning budgets according to task difficulty. Ultimately, our findings provide stable design paradigms for efficiency-driven reinforcement learning.

\section*{Limitations}

\paragraph{Model Scale.} 
Our experiments are conducted using a single model scale (Qwen3-1.7B). Whether the stochastic vulnerabilities and the efficacy of \acoer's mechanisms generalize to larger models (e.g., 7B parameters or more) requires further validation, as increased scale may introduce different optimization dynamics.

\paragraph{Compute Constraints on Multi-Seed Analysis.} 
The stochastic vulnerability of static methods was established using a limited number of initializations. Precisely characterizing the collapse probability across diverse seeds requires extensive experimentation, which is computationally prohibitive at our current scale (requiring approximately 48 GPU-hours per 1,200-step run).

\paragraph{Statistical Significance of Accuracy Gains.} 
While \acoer significantly improves efficiency, its marginal accuracy improvements over stable baselines such as GRPO-acc are not statistically significant given the evaluation set sizes. Detecting a 1\% to 2\% difference reliably would require $n \approx 2{,}000$. Thus, the primary advantage of \acoer lies in its substantial token reduction (efficiency) rather than absolute accuracy superiority.

\paragraph{Cross-Benchmark Variance.} 
Cross-benchmark evaluations reveal a difficulty-dependent efficiency frontier. Although \acoer demonstrates robust accuracy on MATH-500, MATH-Hard, and OlympiadBench, the results on AIME 2025 exhibit minor performance gaps and higher variance, partially due to the dataset's small sample size.

\paragraph{Theoretical Completeness.} 
Our formal analysis (Proposition~1) is exact for a group size of $G=2$. While we extend the intuition to a general $G$ using same-correctness subgroup arguments supported by empirical verification, a complete formal proof for arbitrary group compositions remains an open problem.

\paragraph{Hyperparameter Sensitivity.} 
The \acoer framework introduces multiple hyperparameters (detailed in Appendix~\ref{sec:appendix_acoer}). Although the control-loop architecture is designed to self-adjust and we utilized documented defaults without extensive tuning, the sensitivity of these parameters across varying domains and base models has not been exhaustively explored.

\section*{Ethics Statement}

\paragraph{Potential Risks.}
This work aims to improve computational efficiency, but methods that shorten reasoning traces can be misused to reduce transparency or deployed prematurely in settings where concise outputs mask brittle reasoning. We therefore recommend using \acoer only with task-specific validation, monitoring both accuracy and length distributions, and avoiding safety-critical deployment based solely on token savings.

\paragraph{Artifacts, Licenses, and Intended Use.}
We use publicly available artifacts: Qwen3-1.7B \citep{yang2025qwen3}, NuminaMath-TIR \citep{numina}, MATH-500 and MATH-Hard \citep{math500}, AIME 2025 \citep{aime2025}, OlympiadBench \citep{he2024olympiadbench}, TRL \citep{trl}, Low-Rank Adaptation (LoRA) \citep{hu2022lora}, DeepSpeed Zero Redundancy Optimizer (ZeRO) \citep{zero}, vLLM \citep{vllm}, and \texttt{math\_verify} \citep{mathverify}. Qwen3-1.7B, NuminaMath-TIR, TRL, DeepSpeed, vLLM, and \texttt{math\_verify} are released under Apache-2.0 licenses. We use these artifacts for their intended research, training, and evaluation purposes. Artifacts created by this work, including reward definitions, trained adapters, evaluation logs, and figures, are intended for research comparison rather than direct safety-critical deployment.

\paragraph{Data Content and Privacy.}
We do not collect new user data or human-subject data. The datasets used in this work consist of mathematical problem statements, solutions, and final answers rather than demographic annotations. We checked the dataset schemas and examples used by our pipeline and found no fields intended to name or uniquely identify individual people. We report aggregate metrics only and do not redistribute benchmark text or model-generated traces in the paper.

\paragraph{AI Assistance.}
AI systems were not used to conduct experiments, generate experimental data, or compute reported results. They were used only for LaTeX editing and language polishing.

\bibliography{custom}

\appendix

\section{Proof of Proposition 1}
\label{sec:appendix_proof}

\noindent \textbf{Proof.}
Consider a group of size \(G=2\) with two incorrect completions. Let their lengths be \(\ell_1\) and \(\ell_2\), with \(\ell_1 < \ell_2\). 
Since the incorrect-answer length penalty function \(g(0, \ell)\) is monotonically decreasing with respect to \(\ell\), the corresponding rewards satisfy \(r_1 > r_2\). 

For a group of size \(G=2\), the mean reward is \(\bar{r} = \frac{r_1 + r_2}{2}\). Therefore, the centered reward for the first completion is \(r_1 - \bar{r} = \frac{r_1 - r_2}{2}\). The GRPO advantages for these two completions are calculated as:
\begin{equation}
\begin{aligned}
\hat{A}_1 &= \frac{r_1 - r_2}{2\sigma_r} \\
\hat{A}_2 &= -\hat{A}_1
\end{aligned}
\end{equation}

By applying a first-order Taylor expansion around the mean length \(\bar{\ell} = \frac{\ell_1 + \ell_2}{2}\), the reward for each completion can be approximated as \(r_i \approx g(0, \bar{\ell}) + g'(0, \bar{\ell})(\ell_i - \bar{\ell})\), where \(g'(0, \bar{\ell}) = \left. \frac{\partial g(0, \ell)}{\partial \ell} \right|_{\ell=\bar{\ell}}\).
This linear approximation implies that the standard deviation of the rewards (\(\sigma_r\)) is directly proportional to the standard deviation of the lengths (\(\sigma_\ell\)):
\begin{equation}
\sigma_r = |g'(0, \bar{\ell})| \cdot \sigma_\ell
\end{equation}

Substituting \(\sigma_r\) into the magnitude of the advantage yields:
\begin{equation}
|\hat{A}| = \frac{|g'(0, \bar{\ell})| \cdot |\Delta\ell|}{2 |g'(0, \bar{\ell})| \cdot \sigma_\ell} = \frac{|\Delta\ell|}{2\sigma_\ell}
\end{equation}

Assuming the continuous length signal ensures a non-zero gradient (\(g'(0, \bar{\ell}) \neq 0\)), the \(g'\) terms cancel out. As the length distribution concentrates within the subgroup (\(\sigma_\ell \to 0^+\)), the normalized advantage \(|\hat{A}|\) diverges to infinity.

Conversely, under the correct-only condition (\(\beta=0\)), the penalty function is flat for incorrect answers, meaning \(g(0, \ell) = 0\) for all \(\ell\). Consequently:
\begin{equation}
\begin{aligned}
r_1 &= r_2 \\
\sigma_r &= 0 \\
\hat{A} &= 0
\end{aligned}
\end{equation}

Thus, the advantages vanish entirely for all \(\sigma_\ell\), structurally preventing the divergent feedback loop.

\section{Hyperparameter and Training Details}
\label{sec:appendix_hyperparams}

\begin{table}[h]
\centering
\small
\begin{tabular}{ll}
\toprule
\textbf{Hyperparameter} & \textbf{Value} \\
\midrule
LoRA rank / $\alpha$ & 16 / 32 \\
Learning rate & 5e-6 \\
Per-device batch size & 1 \\
Gradient accumulation & 8 \\
group size ($G$) & 16 \\
Temperature & 1.0 \\
Max completion ($L$) & 8,192 thinking tokens \\
Eval max tokens & 16,384 \\
Training steps & 1,200 \\
Optimizer & AdamW ($\beta_1{=}0.9$, $\beta_2{=}0.999$) \\
Weight decay & 0.01 \\
Warmup ratio & 0.1 \\
DeepSpeed & ZeRO Stage 2 \\
Precision & bfloat16 \\
\bottomrule
\end{tabular}
\caption{Full training hyperparameters. All configurations share identical settings except reward function. LoRA denotes Low-Rank Adaptation.}
\label{tab:hyperparams}
\end{table}

\paragraph{Model and Data.} We train Qwen3-1.7B \citep{yang2025qwen3}, a reasoning model with native \thinktag{}...\thinktagend{} support. The training data comprises the NuminaMath-TIR dataset \citep{numina} (${\sim}$15K English mathematical-reasoning training problems, Apache 2.0). Evaluation benchmarks are treated solely as held-out evaluation sets: MATH-500 \citep{math500} ($n{=}500$, primary evaluation), MATH-Hard \citep{math500} ($n{=}1{,}324$, hard competition problems), AIME 2025 \citep{aime2025} ($n{=}30$, competition-level), and OlympiadBench \citep{he2024olympiadbench} ($n{=}674$, competition math). These artifacts cover mathematical reasoning rather than demographic groups or socially sensitive linguistic phenomena.

\paragraph{Training Protocol.} As detailed in Table~\ref{tab:hyperparams}, all methods share identical training conditions to ensure a controlled comparison: Low-Rank Adaptation (LoRA) \citep{hu2022lora} (rank 16, $\alpha = 32$), learning rate $5{\times}10^{-6}$, GRPO implemented with TRL \citep{trl} with a group size of $G = 16$, temperature 1.0, and maximum completion length of 8,192 tokens. We use gradient accumulation of 8 and DeepSpeed ZeRO-2 \citep{zero}, training for 1,200 steps on NVIDIA A6000 48GB GPUs. Each 1,200-step training run requires approximately 48 GPU-hours, so the 11 reported full training configurations require approximately 528 GPU-hours, excluding preliminary pilot runs. Checkpoints are evaluated every 200 steps using vLLM \citep{vllm} with greedy decoding and 16,384 max tokens. We focus on a single model scale to enable this controlled comparison across configurations; 7B+ validation is left for future work.

\section{\acoer Hyperparameters}
\label{sec:appendix_acoer}

\begin{table}[h]
\centering
\resizebox{0.98\columnwidth}{!}{%
\begin{tabular}{llp{4.5cm}}
\toprule
\textbf{Parameter} & \textbf{Value} & \textbf{Description} \\
\midrule
$\alpha_0$ & 0.02 & Initial efficiency weight \\
$\alpha_{\min}$ & 0.01 & Floor \\
$\alpha_{\max}$ & 0.50 & Ceiling \\
$\alpha_{\text{up}}$ & 1.02 & Multiplicative increase \\
$\alpha_{\text{down}}$ & 0.95 & Multiplicative decrease \\
EMA span & 50 & Smoothing window \\
Check window & 100 & Trend comparison lookback \\
$\delta$ (acc drop) & 0.02 & Threshold for $\alpha$ decrease \\
$\gamma$ (budget ratio) & 0.85 & $B_t = \gamma \cdot \text{EMA}_t$ \\
$B_{\min}$ & 512 & Minimum budget (tokens) \\
Warmup & 200 & Steps before adaptation \\
$k$ (log scale) & 5.0 & Diminishing returns parameter \\
\bottomrule
\end{tabular}
}
\caption{\acoer hyperparameters. All values are defaults used without tuning.}
\label{tab:acoer_hypers}
\end{table}

\begin{table*}[t]
\centering
\small
\begin{tabular}{lp{10cm}c}
\toprule
\textbf{Method} & \textbf{Reward formula} & \textbf{$\beta$ equiv.} \\
\midrule
GRPO (acc only) &
$r = c$ & 0 \\[3pt]

$\beta{=}0$ (linear) &
$r = c + c \cdot 0.3 \cdot (1 - \ell/L)$ & 0 \\[3pt]

\acoer &
$r = c + c \cdot \alpha_t \cdot \frac{\log(1 + k(1 - \ell/B_t))}{\log(1+k)}$ & 0 \\[3pt]

\midrule
$\beta{=}0.01$ &
$r = c + c \cdot 0.3 \cdot (1-\ell/L) - (1-c) \cdot 0.01 \cdot (1-\ell/L)$ & 0.01 \\[3pt]

$\beta{=}0.05$ &
$r = c + c \cdot 0.3 \cdot (1-\ell/L) - (1-c) \cdot 0.05 \cdot (1-\ell/L)$ & 0.05 \\[3pt]

$\beta{=}0.10$ &
$r = c + c \cdot 0.3 \cdot (1-\ell/L) - (1-c) \cdot 0.1 \cdot (1-\ell/L)$ & 0.1 \\[3pt]

GRPO+LP &
$r = c - 0.3 \cdot \ell/L$ \quad (uniform penalty) & 0.3 \\[3pt]

GRPO-LEAD &
$r = e^{-\ell/L}$ (correct), $-1$ (incorrect) & fixed \\[3pt]

ReCut &
$r = \pm 1/|Y|$ \quad (unbounded, bidirectional) & $\infty$ \\[3pt]

\bottomrule
\end{tabular}
\caption{Exact reward formulas for compared methods. $c \in \{0, 1\}$ denotes correctness, $\ell$ denotes reasoning tokens, $L$ is the maximum token limit, and $|Y|$ represents total output tokens. Note that the constant format reward ($r_{\text{format}}$) is added to all formulas.}
\label{tab:reward_formulas}
\end{table*}

\section{Reward Formula Reference}
\label{sec:appendix_rewards}

Table~\ref{tab:reward_formulas} provides exact formulas for all compared methods in our diagnostic framework.

\section{Evaluation Protocol}
\label{sec:appendix_eval}

All evaluations use a standardized pipeline:
\begin{itemize}
    \item \textbf{Prompt}: Single-turn chat with ``Solve the following math problem.'' Model chat template with \texttt{enable\_thinking=True}.
    \item \textbf{Generation}: vLLM \citep{vllm} (bfloat16, greedy decoding, temperature 0.0, max 16,384 tokens). LoRA adapters loaded without merging.
    \item \textbf{Verification}: Answer extraction via regex from \texttt{\textbackslash boxed\{\}}, verified by the \texttt{math\_verify} library \citep{mathverify}.
    \item \textbf{Metrics}: Accuracy, mean/median thinking tokens, accuracy-per-1K-tokens, correct/incorrect token decomposition.
    \item \textbf{Reporting}: Unless otherwise noted, accuracy is reported at the stated checkpoint for a single training run and token counts are means over the evaluation set; collapsed methods report both the final checkpoint and the pre-collapse peak.
\end{itemize}

\section{Training Dynamics}
\label{sec:appendix_training}

\paragraph{Reward variance dynamics.}
Stable methods ($\beta{=}0$) maintain reward variance (\texttt{frac\_reward\_zero\_std} $< 0.35$) throughout training.
Collapsing methods exceed 0.5 approximately 200 steps before accuracy degradation, providing an early warning signal.

\section{Group Composition Analysis}
\label{sec:appendix_group_stats}

With a group size of $G=16$, the probability distribution of group compositions shifts significantly compared to smaller groups. We analyze two scenarios based on the model's accuracy $p$:

\paragraph{Dominance of Mixed Groups ($p=0.7$).} 
On problems of moderate difficulty where the model accuracy $p \approx 0.7$ (incorrect probability $q=0.3$), the probability of generating a mixed group containing at least one correct and one incorrect response is:
\begin{equation}
\begin{aligned}
P(\text{mixed}) &= 1 - p^{16} - q^{16} \\
&\approx 1 - 0.0033 - 0.0000 \\
&\approx 99.7\%
\end{aligned}
\end{equation}
In these groups, the reward standard deviation $\sigma_r$ is determined by the binomial variance $\sqrt{p(1-p)}$. For $p=0.7$, $\sigma_r \approx 0.458$, and for $p=0.8$, $\sigma_r = 0.400$. This confirms that in nearly all training steps, the denominator in GRPO acts as a stable constant $C \approx 0.4$, exerting a \emph{constant downward pressure} on incorrect lengths as described in Section~\ref{sec:mechanism}.

\paragraph{Homogeneous Incorrect Groups on Hard Tasks ($p=0.2$).} 
On challenging problems (e.g., MATH Level 5) where accuracy may drop to $p \approx 0.2$ ($q=0.8$), the probability that all 16 completions are incorrect is:
\begin{equation}
P(\text{all incorrect}) = q^{16} = 0.8^{16} \approx 0.028 \quad (2.8\%)
\end{equation}
While infrequent, these batches occur multiple times within a typical 1,200-step training run. When they occur, $\sigma_r \to 0$, triggering the explosive advantage divergence (Proposition 1) that accelerates the structural collapse.

\section{Scope of Comparison}
\label{sec:appendix_scope}

Our 11 configurations include four published baselines (GRPO-acc, GRPO+LP, GRPO-LEAD, ReCut) and seven systematic ablations designed to isolate the $\beta$ parameter and test adaptive mechanisms.
The $\beta$ sensitivity analysis ($\beta \in \{0, 0.01, 0.05, 0.10\}$) varies only the $\beta$ parameter with fixed $\alpha{=}0.3$, providing the cleanest causal test.

\end{document}